\documentclass[a4paper,11pt,fleqn]{article}
\usepackage{pstricks}
\usepackage{amsmath}
\usepackage{amssymb}
\usepackage{mathrsfs} %pour mathscr
\usepackage{theorem} 
\usepackage{booktabs}
\usepackage{euscript}
\usepackage{exscale,relsize}
\usepackage{graphicx}
\usepackage{booktabs}
\usepackage{srcltx}
\usepackage{xcolor}
\usepackage{charter}

\newcommand{\tciFourier}{{\cal F}}

\newtheorem{theorem}{Theorem}[section]

\newtheorem{lemma}[theorem]{Lemma}

\newtheorem{definition}[theorem]{Definition}

\title{Bounds for Vector-Valued Function Estimation}

\author{
  Andreas Maurer  \\
\small   Adalbertstrasse 55, D-80799 Munchen, Germany \\
\small Email: {\em am@andreas-maurer.eu} \\
\\
Massimiliano Pontil \\
\small Istituto Italiano di Tecnologia,
16163 Genoa, Italy \\
\small Email: {\em massimiliano.pontil@iit.it} \\
\small and \\
\small University College London \\ 
\small Department of Computer Science, London WC1E 6BT, UK 
}

\begin{document}
\maketitle

\begin{abstract}
We present a framework to derive risk bounds for vector-valued learning with
a broad class of feature maps and loss functions. Multi-task learning and
one-vs-all multi-category learning are treated as examples. 
We discuss in detail vector-valued functions with one hidden layer,
and demonstrate that the conditions under which shared representations are
beneficial for multi-task learning are equally applicable to multi-category
learning.
\end{abstract}

\iffalse
%MASSI
TITLE: Rademacher bounds for vector-valued learning with hidden layers

abstract 2-3 more lines improvement over previous work when it exists, cases
which we cover for the first time the framework naturally extended to more
hidden layer (the bound will then have the L^s where s is the number of
hidden layers), but for simplicity here we illustrate only for 1-hidden
later networks

define empirical covariance \fi

\vspace{.5truecm}

\section{Introduction}

The main focus of this paper is to study statistical bounds for (shared)
representation learning under a general class of feature maps and loss
functions. This study is motivated by the development of data-dependent
generalization bounds for multi-category learning with $T$ classes, and for
multi-task learning with $T$ tasks. We show that both problems can be
treated in parallel under a unified framework.

We give bounds on the Rademacher complexity of composite vector-valued
function classes 
\begin{equation*}
\tciFourier \circ \mathcal{G}=\left\{ x\in H\mapsto f\left( g\left( x\right)
\right) \in \mathbb{R}^{T}:f\in \tciFourier ,g\in \mathcal{G}\right\} ,
\end{equation*}%
where the input space $H$ is a finite or infinite dimensional Hilbert space, 
$\mathcal{G}$ is a class of functions (or feature-maps or representations) $%
g:H\rightarrow \mathbb{R}^{K}$, and $\tciFourier $ is a class of output
functions $f:\mathbb{R}^{K}\rightarrow \mathbb{R}^{T}$. Functions in $%
\tciFourier \circ \mathcal{G}$ are chosen on the basis of a finite number $N$
of independent observations and we are interested in uniformly bounding the
incurred estimation errors in terms of the parameters $T$, $K$ and $N$, or
alternatively $n=N/T$, the number of observations per output unit.

There are two main contributions of this work:

\begin{itemize}
\item We provide a common method to derive data dependent bounds for multi-task and
multi-category learning in terms of the complexity of general vector-valued
function classes. In passing we improve on a recent result in \cite{Kloft
2015} on multi-category learning. Our framework is also general enough to be
applied to hybrid coding schemes for multi-category classification such as $%
1 $-vs-$1$ pairwise classification.

\item We apply this method to a large class of vector-valued functions with
shared feature maps to demonstrate that the conditions under which shared
representations are beneficial for multi-task learning are equally
applicable to multi-category learning.
\end{itemize}

Our principal finding is a data-dependent generalization bound, whose
dominant terms have the form 
\begin{equation*}
O\left( \theta \sqrt{\frac{{\rm tr}(\hat{C})}{nT}}\right) +O\left( \theta 
\sqrt{\frac{\lambda _{\max }(\hat{C})}{n}}\right) ,
\end{equation*}%
where $\hat{C}$ is the empirical covariance operator (see below). When
testing multi-task learning we are always told which task we are testing and
thus the relevant component of our vector-valued hypothesis. In the one-vs-all
multi-category setting we of course withhold the identity of the correct class
and thus also of the relevant component. This simple fact is reflected in
the presence of the factor $\theta $, which is one for multi-task learning
and $\sqrt{T}$ for multi-category learning.

Bounds of this form are given for a large class of neural networks with one
hidden layer and rather general nonlinear activation functions, which may
involve inter-unit couplings or intermediate maps to infinite-dimensional
spaces. 
%\massi{I would emphasis more the activation function, because the mapping to infinite-dimensional spaces is somehow less practical}. 
A similar bound also holds for linear classes with trace-norm constraints,
which can also be interpreted as composite classes, see e.g. \cite{Srebro}.

As $T$ increases the second term dominates the above expression. This term
however depends only on the largest eigenvalue, instead of the trace, of the
empirical covariance. If $T$ is large and the data is high-dimensional the
intermediate representation can therefore give a considerable advantage.
This has been established for multi-task learning in several works and, as
we show here, holds equally for multi-category learning, in agreement with previous 
empirical studies of the benefit of trace-norm regularization in multi-category learning \cite{Amit 2007}.

In Section \ref{sec:2} we explain how the complexities of multi-category and
multi-task learning can be reduced to the complexities of vector-valued
function classes and bounded by a common expression. We briefly discuss
independent and linear classes in Section \ref{sec:3.1} and \ref{sec:3.2}.
Then in Section \ref{sec:3.3}, we present our principal result on
nonlinear composite classes. The appendix contains statements and proofs of our results in their most general form.

\subsection{Previous Work}

Bounds for multi-layered networks are given in the now classical work \cite%
{Anthony 2009} in terms of covering numbers. More recently there are bounds
using Rademacher averages \cite{Neyshabur 2015}. These works mainly consider
scalar outputs and ignore the regularizing effects of intermediate
representations.

Early work to consider the potential benefits of shared representations was
in the setting of multi-task learning and learning to learn \cite{Baxter
2000}. Subsequent work has focused more on learning bounds for linear
feature learning \cite{Cesa,Maurer 2006}. Recently \cite{Maurer 2015}
presented a general bound for multi-task representation learning. 
%, building upon a general method to bound the Gaussian complexity of composite function classes in \cite{Maurer 2014}. 
Although there has been substantial work on the statistical analysis of
learning shared representations for multi-task learning, less has been done
for multi-category learning. This is in contrast with the large body of
empirical work on deep networks, which are often trained with a multi-class
loss \cite{Darrell}, such as the soft max or multi-class hinge loss. In
this work we close this gap. 
%In this work, we demonstrate that the conditions under which shared representations are beneficial for multi-task learning are equally applicable to multi-category learning. 
%In this work we close this gap by providing a unified framework to derive statistical bounds for general vector-valued representation learning methods, which include multi-category learning and multi-task learning as special cases.}

\section{Multi-Category and Multi-Task Learning\label{sec:2}}

We extend the notion of Rademacher complexity to the vector-valued setting.

\begin{definition}
\label{Definition complexity} Let $T,N\in 
%TCIMACRO{\U{2115} }%
%BeginExpansion
\mathbb{N}
%EndExpansion
$, let $\mathcal{X}$ be any set, $\tciFourier $ a class of functions $f:%
\mathcal{X\rightarrow 
%TCIMACRO{\U{211d} }%
%BeginExpansion
\mathbb{R}
%EndExpansion
}^{T}$, $\mathbf{x}=\left( x_{1},\dots,x_{N}\right) \in \mathcal{X}^{N}$,
and let $I:\left\{ 1,\dots,T\right\} \rightarrow 2^{\left\{
1,\dots,N\right\} }$ be a function which assigns to every $t\in \left\{
1,\dots,T\right\} $ a subset $I_{t}\subset \left\{ 1,\dots,N\right\} $. We
define 
\begin{equation*}
R_{I}\left( \tciFourier ,\mathbf{x}\right) =\frac{1}{N}\mathbb{E}\sup_{f\in 
\mathcal{\tciFourier }}\sum_{t=1}^{T}\sum_{i\in I_{t}}\epsilon
_{ti}f_{t}\left( x_{i}\right) ,
\end{equation*}%
where the $\epsilon _{ti}$ are doubly indexed, independent Rademacher
variables (uniformly distributed on $\left\{ -1,1\right\} $).
\end{definition}

In this section we show that the estimation problem for both multi-category
and multi-task learning can be reduced to the problem of bounding $%
R_{I}\left( \tciFourier ,\mathbf{x}\right)$ for appropriate choices of the
function $I$.

\subsection{Multi-Category Learning}

{Let $C \in {\mathbb{N}}$ be the number of categories}. There is an unknown
distribution $\mu $ on $H\times \left\{ 1,\dots,C\right\} $, a
classification rule $cl:\mathcal{\mathbb{R}}^{T}\rightarrow \left\{
1,\dots,C\right\} $, and for each label $y\in \left\{1,\dots,C\right\}$ a
surrogate loss function $\ell _{y}:\mathbb{R} ^{T}\rightarrow \mathbb{R}_{+}$%
. The loss function $\ell _{y}$ is designed so as to upper bound or
approximate the indicator function of the set $\left\{ z\in \mathbb{R}
^{T}:cl\left( z\right) \neq y\right\} $. Here we consider the simple case,
where $T=C$. For the construction of appropriate loss functions see \cite{Crammer 2002,Kloft 2015,Rosasco 2012}. These loss functions are Lipschitz on $%
\mathbb{R} ^{T}$ relative to the Euclidean norm, with some Lipschitz
constant $L_{\mathrm{mc}}$, often interpretable as an inverse margin.

Given a class $\tciFourier $ of functions $f:H\rightarrow \mathbb{R}^{T}$ we
want to find $f\in \tciFourier $ so as to approximately minimize the
surrogate risk 
\begin{equation*}
\mathbb{E}_{\left( x,y\right) \sim \mu }\ell _{y}\left( f\left( x\right)
\right) .
\end{equation*}%
Since we do not know the distribution $\mu $, this is done on the basis of a
sample of $N=nT$ observations $\left( \mathbf{x,y}\right) =\left( \left(
x_{1},y_{1}\right) ,\dots ,\left( x_{N},y_{N}\right) \right) \in \left(
H\times \{1,\dots ,C\}\right) ^{N}$, drawn i.i.d. from the distribution $\mu 
$. We then solve the problem 
\begin{equation*}
\hat{f}=\arg \min_{f\in \tciFourier }\frac{1}{N}\sum_{i=1}^{N}\ell
_{y_{i}}\left( f\left( x_{i}\right) \right) .
\end{equation*}%
To give a performance guarantee for $\hat{f}$ we would like to know how far
the empirical minimum above is from the true surrogate risk of $\hat{f}$.
This difference is upper bounded by 
\begin{equation*}
\sup_{f\in \tciFourier }\left[ \mathbb{E}_{\left( x,y\right) \sim \mu }\ell
_{y}\left( f\left( x\right) \right) -\frac{1}{N}\sum_{i=1}^{N}\ell
_{y_{i}}\left( f\left( x_{i}\right) \right) \right].
\end{equation*}
It is by now well known (see e.g.~\cite{Bartlett 2002}) that the above
expression has, with high probability in the sample, a bound, whose dominant
term is given by 
\begin{equation}
\frac{2}{N}~\mathbb{E}\sup_{f\in \mathcal{\tciFourier }}\sum_{i=1}^{N}
\epsilon _{i}\ell _{y_{i}}\left( f\left( x_{i}\right) \right)\,,
\label{eq:hhhh}
\end{equation}
where the $\epsilon _{i}$ are independent Rademacher (uniform $\left\{
-1,1\right\} $-distributed) variables. We now apply the following result 
\cite[ Corollary 6]{Maurer 2016}.

\begin{theorem}
\label{Theorem Contraction}Let $\mathcal{X}$ be any set, $\left(
x_{1},\dots,x_{n}\right) \in \mathcal{X}^{n}$, let $\tciFourier $ be a class
of functions $f:\mathcal{X}\rightarrow \mathbb{R}^{T}$and let $h_{i}: 
\mathbb{R}^{T}\rightarrow \mathbb{R}$ have Lipschitz norm bounded by $L$.
Then 
\begin{equation*}
\mathbb{E}\sup_{f\in \tciFourier }\sum_{i=1}^n\epsilon _{i}h_{i}\left(
f\left( x_{i}\right) \right) \leq \sqrt{2}L\mathbb{E}\sup_{f\in \tciFourier
}\sum_{t,i}\epsilon _{ti}f_{t}\left( x_{i}\right) ,
\end{equation*}
where $\epsilon _{ti}$ is an independent doubly indexed Rademacher sequence
and $f_{t}$ is the $t$-th component of $f$.
\end{theorem}

Using this theorem and the Lipschitz property of the loss functions $\ell
_{y_{i}}$, we upper bound \eqref{eq:hhhh} by 
\begin{equation}
\frac{2 \sqrt{2}}{N}L_{\mathrm{mc}}~\mathbb{E}\sup_{f\in \tciFourier
}\sum_{t=1}^{T}\sum_{i=1}^{N}\epsilon _{ti}f_{t}\left( x_{i}\right) .
\label{yytt}
\end{equation}
A similar argument can be based on Slepian's inequality with a passage to
Gaussian complexities \cite{Kloft 2015}. In this case the $\epsilon _{ti}$
have to be replaced by independent standard normal variables $\gamma _{ti}$,
and $\sqrt{2}$ replaced by $\sqrt{\pi /2}$. The approach chosen here is
simpler and allows us to improve some results of \cite{Kloft 2015} in the
linear case. For our final result (Theorem \ref{Theorem composite class}
below) however we also need Gaussian complexities.

We define $I^{\mathrm{mc}}:\left\{ 1,\dots ,T\right\} \rightarrow 2^{\left\{
1,\dots ,N\right\} }$ by $I_{t}^{\mathrm{mc}}=\left\{ 1,\dots ,N\right\} $
for all $t\,$. With Definition \ref{Definition complexity} the quantity %
\eqref{yytt} then becomes 
\begin{equation}
2\sqrt{2}L_{\mathrm{mc}}R_{I^{\mathrm{mc}}}\left( \tciFourier ,\mathbf{x}
\right) .  \label{MCL complexity}
\end{equation}

\subsection{Multi-Task Learning}

In this setting there is an output space $\mathcal{Y}$, and for each task $t\in \left\{
1,\dots ,T\right\} $ a distribution $\mu _{t}$ on $H\times \mathcal{Y}$ and
a loss function $\ell _{t}:\mathbb{R}\times \mathcal{Y\rightarrow }\left[
0,1 \right] $, which is assumed to be Lipschitz with constant at most $L_{ 
\mathrm{mt}}$ in the first argument for every value of the second. Given a
class $\tciFourier $ of functions $f:H\rightarrow \mathbb{R}^{T}$ we want to
find $f\in \tciFourier $ so as to approximately minimize the task-average
risk 
\begin{equation*}
\frac{1}{T}\sum_{t=1}^{T}\mathbb{E}_{\left( x,y\right) \sim \mu _{t}}\ell
_{t}\left( f_{t}\left( x\right) ,y\right) ,
\end{equation*}
where $f_{t}$ is the $t$-th component of the function $f$. For each task $t$
there is a sample $\left( \mathbf{x}_{t},\mathbf{y}_{t}\right) =\left(
\left( x_{t1},y_{t1}\right) ,\dots ,\left( x_{tn},y_{tn}\right) \right) $
drawn i.i.d. from $\mu _{t}$. One solves the problem 
\begin{equation*}
\hat{f}=\arg \min_{f\in \tciFourier }\frac{1}{nT}\sum_{t=1}^{T}
\sum_{i=1}^{n}\ell _{t}\left( f_{t}\left( x_{ti}\right) ,y_{i}\right) .
\end{equation*}
As before we are interested in the supremum of the estimation difference 
\begin{equation*}
\sup_{f\in \tciFourier }\frac{1}{T}\sum_{t=1}^{T}\left[ \mathbb{E}_{\left(
x,y\right) \sim \mu _{t}}\ell _{t}\left( f_{t}\left( x\right) ,y\right) - 
\frac{1}{n}\sum_{i=1}^{n}\ell _{t}\left( f_{t}\left( x_{ti}\right)
,y_{i}\right) \right] .
\end{equation*}
As shown in \cite{Ando 2005} or \cite{Maurer 2006} there is again a high
probability bound, whose dominant term is given by the vector-valued
Rademacher complexity 
\begin{equation*}
\frac{2}{nT}\mathbb{E}\sup_{f\in \tciFourier
}\sum_{t=1}^{T}\sum_{i=1}^{n}\epsilon _{ti}\ell _{t}\left( f_{t}\left(
x_{ti}\right) ,y_{i}\right) \leq \frac{2}{nT}L_{\mathrm{mt}}~\mathbb{E}
\sup_{f\in \tciFourier }\sum_{t=1}^{T}\sum_{i=1}^{n}\epsilon
_{ti}f_{t}\left( x_{ti}\right) ,
\end{equation*}
where we eliminated the Lipschitz functions with a standard contraction
inequality as in \cite{Meir}. We now collect all the tasks input samples $%
\mathbf{x}_{t}$ in a big sample $\mathbf{x}=\left( x_{1},\dots ,x_{N}\right)
\in H^{N}$ with $N=nT$, and define $I^{\mathrm{mt}}:\left\{ 1,\dots
,T\right\} \rightarrow 2^{\left\{ 1,\dots ,N\right\} }$ so that $I_{t}^{\mathrm{%
mt }}$ is the set of all indices of the examples for task $t$. Thus $\mathbf{%
x} _{t}=\left( x_{i}\right) _{i\in I_{t}}$ and $n=\left\vert I_{t}^{\mathrm{%
mt} }\right\vert $. The right hand side above again becomes 
\begin{equation}
2 L_{\mathrm{mt}}R_{I^{\mathrm{mt}}}\left( \tciFourier ,\mathbf{x} \right) .
\label{Principal complexity}
\end{equation}

\subsection{A Common Expression to Bound}

Comparing (\ref{MCL complexity}) and (\ref{Principal complexity}) we can
summarize: Let $\mathcal{\tciFourier }$ be a class of functions with values
in $\mathbb{R}^{T}$. The empirical Rademacher complexity of $\mathcal{%
\tciFourier }$ as used in multi-category learning and the empirical
Rademacher complexity of $\mathcal{\tciFourier }$ as used in multi-task
learning are up to (Lipschitz-) constants, bounded by $R_{I}\left(
\tciFourier ,\mathbf{x}\right) $, where the function $I$ is either $I^{%
\mathrm{mc}}$ in the multi-category case or $I^{\mathrm{mt}}$ in the
multi-task case and $I_{t}^{\mathrm{mc}}=\left\{ 1,\dots ,N\right\} $ while $%
I_{t}^{\mathrm{mt}}\subseteq \left\{ 1,\dots ,N\right\} $ is the set of
indices of examples for {task} $t$.

With appropriate definitions of the function $I$, bounds on $R_{I}\left(
\tciFourier ,\mathbf{x}\right) $ also lead to learning bounds in hybrid
situations where there are several multi-category tasks, potentially with
classes occurring in more than one task. In the case of 1-vs-1
voting schemes $T=C\left( C-1\right) /2$, so there is a
component for every unordered pair of distinct classes $\left(
c_{1},c_{2}\right) $. Then we define a $I_{\left( c_{1},c_{2}\right) }$ to
be the set of indices of all examples for the classes $c_{1}$ and $c_{2}$. 

In general $I_{t}$ should be the set indices of those examples, which occur
as arguments of $f_{t}$ in the expression of the empirical error. For
reasons of space however we will stay with the cases of multi-task and
1-vs-all multi-category learning as explained above. We refer to the
appendix for the most general statements of our results.

To lighten notation we write $R_{I^{\mathrm{mc}}}=R_{\mathrm{mc}}$ and $%
R_{I^{\mathrm{mt}}}=R_{\mathrm{mt}}$. We also use the notation $R_{\alpha }$%
, where the variable $\alpha $ can be either ``mc'' or ``mt''. It will also
be useful to observe that for $\left( a_{1},\dots ,a_{N}\right) \in \mathbb{R%
}^{N}$ 
\begin{equation*}
\sum_{t=1}^{T}\sum_{i\in I_{t}^{\alpha }}a_{i}=\theta _{\alpha
}^{2}\sum_{i=1}^{N}a_{i},
\end{equation*}%
where $\theta _{\mathrm{mc}}=\sqrt{T}$ and $\theta _{\mathrm{mt}}=1$.

\section{Specific Bounds}

\label{sec:3} We show how the quantity $R_{\alpha }\left( \tciFourier ,%
\mathbf{x}\right) $ may be bounded, first by a simple and general method of
reduction to the Rademacher complexities of scalar function classes, then
for certain linear classes, and finally we state and prove our main results
for composite classes.

\subsection{Component Classes and Independent Learning}

\label{sec:3.1} Given a class $\mathcal{\tciFourier }$ of functions with
values in $\mathbb{R}^{T}$ we can define for each $t\in \left\{
1,\dots,T\right\} $ the scalar valued component class $\mathcal{\tciFourier }%
_{t}=\left\{ f_{t}:f\in \tciFourier \right\} $. By bringing the supremum
inside the first sum in ( \ref{Principal complexity}) we obtain the bound 
\begin{equation*}
R_{I}\left( \tciFourier ,\mathbf{x}\right) \leq \frac{1}{N}
\sum_{t=1}^{T}\mathbb{E}\sup_{f\in \mathcal{\tciFourier }_{t}}\sum_{i\in
I_{t}}\epsilon _{i}f\left( x_{i}\right) ,
\end{equation*}
which is just a sum of standard, scalar case, empirical Rademacher averages.

In the case of independent learning the components of the members of $%
\mathcal{\tciFourier }$ are chosen independently, so that
\[\mathcal{\
\tciFourier =}\prod_{t}\mathcal{\tciFourier }_{t}=\left\{ \left(
f_{1},\dots,f_{T}\right) :\forall t,f_{t}\in \mathcal{\tciFourier }
_{t}\right\}\,,
\] and the above bound becomes an identity and unimprovable. In
most cases $\mathbb{E}\sup_{f\in \mathcal{\tciFourier }_{t}}\sum_{i\in
I_{t}}\epsilon _{i}f\left( x_{i}\right) $ is of the order $\sqrt{\left\vert
I_{t}\right\vert }$ so the above implies a bound of the order $\theta _{\alpha
}/ \sqrt{N}$.

\subsection{Linear Classes}

Before proceeding we require some more notation. Given a
sequence of input vectors, $( x_{1},\dots,x_{N}) \in H^{N}$ we define the
empirical covariance operator $\hat{C}$ by%
\begin{equation*}
\langle \hat{C}\ v,w\rangle =\frac{1}{N}\sum_{i=1}^N\langle v,x_{i}\rangle
\langle x_{i},w\rangle \text{ for~every~}v,w\in H\,.
\end{equation*}%
Furthermore, given a function $I:\left\{ 1,\dots,T\right\} \rightarrow
2^{\left\{ 1,\dots,N\right\} }$, we define the empirical covariance operator 
$\hat{C}_t$ by 
\begin{equation*}
\langle \hat{C}_t v,w\rangle =\frac{1}{|I_t|}\sum_{i\in I_t}\langle
v,x_{i}\rangle\langle x_{i},w\rangle \text{.}
\end{equation*}

\label{sec:3.2} We consider linear transformations $W:H\rightarrow \mathbb{R}%
^{T}$ of the form 
\begin{equation*}
x\mapsto \left( \left\langle w_{1},x\right\rangle ,\dots,\left\langle
w_{T},x\right\rangle \right)
\end{equation*}
with weight-vectors $w_{t}\in H$. Corresponding function classes will be
defined by constraints on the norms of such transformations. We use the
mixed $(2,p)$-norms which are defined as 
\begin{equation*}
\left\Vert W\right\Vert _{2,p}=\big\|(\| w_{1}\|,\dots,\| w_{T}\|)\big\|_p
\end{equation*}
and the trace norm $\left\Vert \cdot \right\Vert _{{\rm tr}}=%
{\rm tr}\big( 
\sqrt{W^* W }\big) $. The norm $\left\Vert \cdot \right\Vert _{2,2}$ is also
known as the Hilbert-Schmidt norm or, for finite-dimensional $H$, as the
Frobenius norm $\Vert W \Vert _{2,2}=\sqrt{ \sum_{t}\Vert
w_{t}\Vert ^{2}}$. For $B>0$ we consider the classes, 
\begin{equation*}
\mathcal{W}_{2,p}=\left\{ W:\left\Vert W\right\Vert _{2}\leq BT^{1/p}\right\}
\end{equation*}
and 
\begin{equation*}
\mathcal{W}_{{\rm tr}}=\left\{ W:\left\Vert W\right\Vert _{{\rm tr}%
}\leq B\sqrt{T}\right\} .
\end{equation*}
The class $\mathcal{W}_{{\rm tr}}$ can be defined alternatively as $%
\mathcal{W} _{{\rm tr}}=\left\{ VW:W\in \mathcal{W},V\in \mathcal{V}%
\right\} $, where $\mathcal{W=}\{ W:H\rightarrow \mathbb{R} ^{T},\left\Vert
W\right\Vert _{2,2}\leq 1\} $ and $\mathcal{V=}\{ V:\mathbb{R}%
^{T}\rightarrow \mathbb{R}^{T},\left\Vert V\right\Vert _{2,2}\leq B\sqrt{T}%
\} $, see for example \cite{Srebro} and references therein. This exhibits $%
\mathcal{W}_{{\rm tr}}$ as a composite vector-valued function class.

The factor $T^{1/p}$ in the definition of $\mathcal{W}_{2,p}$ is essential
when discussing the dependence on $T$. If it were absent then by Jensen's
inequality the average norm allowed to the weight vectors would be bounded
by $B/T^{1/p}$, so the class is regularized to death as $T$ increases. This
applies in particular to the case of multi-category learning, where each
component needs to be able to win over all the others by some margin. The
same argument applies to the $\sqrt{T}$ in the constraint of the trace-norm
class. In this sense it is not quite correct to speak of rates in $T$ if the
constraint on the norm is held constant as in \cite{Kloft 2015}.

For simplicity we assume that $\left\Vert x_{i}\right\Vert =1$ for all $i$
(as with a Gaussian RBF-kernel) for the rest of this subsection. Note that this
implies tr$(\hat{C})={\rm tr}(\hat{C}_{t})=1$. We also consider only the cases of
multi-category and multi-task learning. Statements and proofs for general
index sets $I_{t}$ and general values of the $\left\Vert x_{i}\right\Vert $
are given in the appendix. We first give some lower and upper bounds for $%
\mathcal{W}_{2,\infty }$ and $\mathcal{W}_{2,p}$.

\begin{theorem}
\label{Lemma lower and upper bounds frobenius}For $p\in \lbrack 2,\infty ]$%
\begin{equation*}
B~\theta _{\alpha }\sqrt{\frac{1}{2n}}\leq R_{\alpha }\left( \mathcal{W}%
_{2,\infty },\mathbf{x}\right) \leq R_{\alpha }\left( \mathcal{W}_{2,p},%
\mathbf{x}\right) \leq B~\theta _{\alpha }\sqrt{\frac{1}{n}}
\end{equation*}%
and for $p\in \lbrack 1,2]$ and $1/p+1/q=1$%
\begin{equation*}
R_{\alpha }\left( \mathcal{W}_{2,2},\mathbf{x}\right) \leq R_{\alpha }\left( 
\mathcal{W}_{2,p},\mathbf{x}\right) \leq 2^{1/q}B~\theta _{\alpha }\sqrt{%
\frac{q}{n}}.
\end{equation*}
\end{theorem}

The lower bound in the $[2,\infty ]$-regime is simply $1/\sqrt{2}$ times the
upper bound. If we set $\Lambda =T^{1/p}B$, then the multi-category bound
for the $\left[ 1,2\right] $-regime can be compared to the one given in \cite%
{Kloft 2015}, which is larger by a factor of $O\sqrt{q}$. This improvement
is however exclusively due to our trick of staying with Rademacher variables
when eliminating the loss functions.

The norms in the lemma above are not very useful for multi-task learning, as
the bounds show no improvement as the number of tasks increases. This is
different for the trace-norm constrained class $\mathcal{W}_{{\rm tr}}$,
for which we have the following result, which already exhibits a typical
behaviour of composite classes. The proof of a more general version is given in the
appendix.

\begin{theorem}
\label{Theorem Tracenorm}%
\begin{equation*}
R_{\alpha }\left( \mathcal{W}_{tr},\mathbf{x}\right) \leq B~\theta _{\alpha
}\left( \sqrt{\frac{2\left( \ln \left( nT\right) +1\right) }{nT}}+\sqrt{%
\frac{\lambda _{\max }(\hat{C})}{n}}\right)\,.
\end{equation*}
\end{theorem}

If we divide this bound by the above lower bound for
regularization with the Hilbert Schmidt norm, we obtain 
\begin{equation*}
\frac{R_{\alpha }\left( \mathcal{W}_{{\rm tr}},\mathbf{x}\right) }{%
R_{\alpha }\left( \mathcal{W}_{2,2},\mathbf{x}\right) }\leq 2\sqrt{\frac{\ln
\left( nT\right) +1}{T}}+\sqrt{\frac{2\lambda _{\max }(\hat{C})}{{\rm tr}(%
\hat{C})}},
\end{equation*}%
a quotient, which highlights the potential benefits of composite classes. As 
$T$ increases the second term becomes dominant. The quotient $\lambda _{\max
}(\hat{C})/{\rm tr}(\hat{C})$ can be seen as the inverse of an effective
data-dimension. Indeed for whitened data ${\rm tr}(\hat{C})=d~\lambda
_{\max }(\hat{C})$, if $d$ is the number of nonzero eigenvalues of $\hat{C}$%
. The relative estimation benefit of the intermediate representation
increases with the number $T$ of classes or tasks and with the effective
dimensionality of the data. This appears to be a rather general feature of
composite vector-valued classes, also in the nonlinear case.

\subsection{Composite Classes and Representation Learning}

\label{sec:3.3} We now consider function classes $\mathcal{V}\circ \phi
\circ \mathcal{W}$ of the form

\begin{center}
\begin{tabular}{lllllll}
& $\mathcal{W}$ &  & $\phi $ &  & $\mathcal{V}$ &  \\ 
$x\in H$ & $\longrightarrow $ & $\mathbb{R} ^{K}$ & $\longrightarrow $ & $%
H^{\prime }$ & $\rightarrow $ & $\mathbb{R}^{T}$\,.%
\end{tabular}
\end{center}

Here inputs $x\in H$ are first mapped to $\mathbb{R}^{K}$ by a linear
function $W$ from a class $\mathcal{W}$. The vector$\ Wx$ is then mapped to
another Hilbert-space $H^{\prime }$ by a fixed Lipschitz feature map $\phi :%
\mathbb{R}^{K}\rightarrow H^{\prime }$. Finally $\phi \left( Wx\right) $ is
mapped to the $T$-dimensional vector $V\phi \left( Wx\right) $ by the linear
map $V$ chosen from $\mathcal{V}$.

For $W\in \mathcal{W}$ we consider the constraints $\left\Vert W\right\Vert
_{2,\infty }\leq b_{\infty }$, $\left\Vert W\right\Vert _{2,2}\leq b_{2}$
and $\left\Vert W\right\Vert _{2,1}\leq b_{1}$, denoting the respective
classes by $\mathcal{W}_{2,\infty }$, $\mathcal{W}_{2,2}$, and $\mathcal{W}
_{2,1}$. For $\mathcal{V}$ we take the
constraint $\left\Vert V\right\Vert _{2,\infty }\leq a$. This choice allows
us to vary $T$ and keep $a$ fixed at the same time. For the ``activation
function" $\phi $ we assume a Lipschitz constant $L_{\phi }$. We make the
simplifying assumption that $\phi \left( 0\right) =0$.

The function $\phi $ makes the model quite general. Suppose first that $%
H^{\prime }=\mathbb{R}^{K}$. If $\phi $ is the identity function we obtain a
linear class, defined through its factorization, much like the case of
trace-norm regularization discussed earlier. If the components of $\phi $
are sigmoids or the popular rectilinear activation functions, we obtain a
rather standard neural network with hidden layer, but $\phi $ could also
include inter-unit interactions, such as poolings or lateral inhibitions (see, e.g. \cite{Haykin,Lecun}) as long as it observes the Lipschitz condition. 

However, the dimension of $H^{\prime}$ need not be $K$ and $\phi $ could be
defined by a radial basis function network with fixed centers or it could
also be the feature-map induced by some kernel on $\mathbb{R} ^{T}$, say a
Gaussian kernel of width $\Delta $, in which case $L_{\phi }=2/\Delta$. To
enforce $\phi \left( 0\right) =0$ we need to translate the original feature
map $\psi $ of the Gaussian kernel as $\phi \left( x\right) =\psi \left(
x\right) -\psi \left( 0\right) $.

Here the underlying assumption is, that there is a common $K$-dimensional
representation of the data in which the data has sufficient separation
properties, but the separating functions may be highly nonlinear.\bigskip

\begin{theorem}
\label{Theorem composite class}There are universal constants $c_{1}$ and $%
c_{2}$ such that under the above conditions%
\begin{eqnarray*}
R_{\alpha }\left( \mathcal{V\phi }\left( \mathcal{W}_{2,\infty }\right) ,%
\mathbf{x}\right) &\leq &L_{\phi }ab_{\infty }\theta _{\alpha }\left( c_{1}K%
\sqrt{\frac{{\rm tr}( \hat{C}) }{nT}}+c_{2}\sqrt{\frac{K\lambda _{\max }( 
\hat{C}) }{n}}\right) \\
R_{\alpha }\left( \mathcal{V\phi }\left( \mathcal{W}_{2,2}\right) ,\mathbf{x}%
\right) &\leq &L_{\phi }ab_{2}\theta _{\alpha }\left( c_{1}\sqrt{\frac{K~%
{\rm tr}( \hat{C}) }{nT}}+c_{2}\sqrt{\frac{\lambda _{\max }\left( \hat{C}%
\right) }{n}}\right) \\
R_{\alpha }\left( \mathcal{V\phi }\left( \mathcal{W}_{2,1}\right) ,\mathbf{x}%
\right) &\leq &L_{\phi }ab_{1}\theta _{\alpha }\left( c_{1}\sqrt{\frac{2%
{\rm tr}( \hat{C}) +8\lambda _{\max }( \hat{C}) \ln K}{nT}}+c_{2}\sqrt{%
\frac{\lambda _{\max }( \hat{C}) }{n}}\right)\,.
\end{eqnarray*}%
\bigskip
\end{theorem}

We highlight some implications of the above theorem.
\begin{enumerate}
\item The bounds differ in their dependence on the dimension $K$ of the hidden
layer which is linear, radical and logarithmic respectively. {F}or $\mathcal{%
W}_{2,1}$ the dependence on $K$ is logarithmic and scales only with $\lambda
_{\max }(\hat{C})$.

\item In the case of multi-task learning with $\mathcal{W}_{2,2}$ and $\mathcal{W}%
_{2,1}$ the dependence on $K$ vanishes in the limit $T\rightarrow \infty $. 
In this limit the first term in parenthesis vanishes in all three cases,
leaving only the second term.

\item Multi-category learning requires more data with $\theta =\sqrt{T}$, but if
we take a simultaneous limit in $T$ and $n$ such that $T/n$ remains bounded,
then the behaviour is the same as for multi-task learning with $T\rightarrow
\infty $.

\item In both cases the second term becomes dominant for large $T$. For the first
bound crudely setting $\lambda _{\max }( \hat{C}) =1/d$ this term scales
with $\sqrt{K/d}$ and exhibits the benefit of the shared representation as
that of dimensional reduction. A similar interpretation holds for the other
bounds with some implicit dependence of $b_{2}$ and $b_{1}$ on the dimension
of the representation.
\end{enumerate}

The proof uses the following recent result on the expected suprema of
Gaussian processes \cite{Maurer 2014}. For a set $Y\subseteq \mathbb{R} ^{m}$
the Gaussian width $G\left( Y\right) $ is defined as 
\begin{equation*}
G\left( Y\right) =\mathbb{E}\sup_{y\in Y}\left\langle \gamma ,y\right\rangle
=\mathbb{E}\sup_{y\in Y}\sum_{i=1}^{m}\gamma _{i}y_{i},
\end{equation*}%
where $\gamma =\left( \gamma _{1},\dots,\gamma _{m}\right) $ is a vector of
independent standard normal variables.

\begin{theorem}
\label{Theorem Chain Rule1}Let $Y\subseteq \mathbb{R}^{n}$ have (Euclidean)
diameter $D\left( Y\right) $ and let $\tciFourier $ be a class of functions $%
f:Y\rightarrow \mathbb{R}^{m}$, all of which have Lipschitz constant at most 
$L\left( \tciFourier \right) $. Let $\tciFourier \left( Y\right) =\left\{
f\left( y\right) :f\in \tciFourier ,y\in Y\right\} $. Then for any
\thinspace $y_{0}\in Y$ 
\begin{equation}
G\left( \tciFourier \left( Y\right) \right) \leq c_{1}L\left( \tciFourier
\right) G\left( Y\right) +c_{2}D\left( Y\right) Q\left( \tciFourier \right)
+G\left( \tciFourier \left( y_{0}\right) \right) ,  \label{eq:jjj}
\end{equation}%
where $c_{1}$ and $c_{2}$ are universal constants and 
\begin{equation*}
Q\left( \tciFourier \right) =\sup_{\mathbf{y},\mathbf{y}^{\prime }\in Y,~%
\mathbf{y}\neq \mathbf{y}^{\prime }}\mathbb{E}\sup_{f\in \tciFourier }\frac{%
\left\langle \mathbf{\gamma },f\left( \mathbf{y}\right) -f\left( \mathbf{y}%
^{\prime }\right) \right\rangle }{\left\Vert \mathbf{y}-\mathbf{y}^{\prime
}\right\Vert }.
\end{equation*}
\end{theorem}

We refer to the appendix for statement and proof of a more general version going beyond 1-vs-all
multi-category and multi-task learning.

\noindent {\it Idea of proof for Theorem \protect\ref{Theorem composite class}.}
We use Theorem \ref{Theorem Chain Rule1} by setting 
\[
Y=\left\{ W\mathbf{x} =\left( \langle w_{k},x_{i}\rangle \right) _{k\leq K,~i\leq N}:W\in \mathcal{
W}\right\} \subseteq \mathbb{R}^{KN}
\]
where $\mathcal{W}$ will be either $%
\mathcal{W}_{2,\infty }$, $\mathcal{\ W}_{2,2}$ or $\mathcal{W}_{2,1}$. Note that
the cardinality $\left\vert I_{t}\right\vert $ is either $N$ or $n$ in the
cases considered here. For $\tciFourier $ we take the set of functions 
\begin{equation*}
\left\{ \left( y_{ki}\right) \in \mathbb{R}^{KN}\mapsto \left( \left\langle
v_{t},\phi \left( y_{i}\right) \right\rangle \right) _{t\leq T,i\in
I_{t}}\in \mathbb{R}^{T\left\vert I_{t}\right\vert }:v\in \mathcal{V}%
\right\} 
\end{equation*}%
restricted to $Y$, so $\tciFourier \left( Y\right) $ is a subset of $\mathbb{%
R}^{T^{2}n}$ for multi-category and $\mathbb{R}^{Tn}$ for multi-task learning.
This again accounts for the additional factor of $\sqrt{T}$ for the
complexity of multi-category learning. By a well known bound on Rademacher
averages in terms of Gaussian averages \cite{Ledoux Talagrand 1991} 
\begin{eqnarray}
\nonumber
\mathbb{E}\sup_{W\in \mathcal{V},W\in \mathcal{W}}\sum_{t}\sum_{i\in
I_{t}}\epsilon _{ti}V\phi \left( Wx_{i}\right)&  \leq &\sqrt{\frac{\pi }{2}}%
\mathbb{E}\sup_{W\in \mathcal{V},W\in \mathcal{W}}\sum_{t}\sum_{i\in
I_{t}}\gamma _{ti}V\phi \left( Wx_{i}\right) \\
& = &\sqrt{\frac{\pi }{2}}G\left(
\tciFourier \left( Y\right) \right) .  \label{Bound by Gaussian width}
\end{eqnarray}
To bound $G\left( \tciFourier \left( Y\right) \right) $ we then just need to
bound the individual components of the right hand side of equation %
\eqref{eq:jjj}, namely the largest Lipschitz constant $L\left( \tciFourier
\right) $, the differential Gaussian width $Q\left( \tciFourier \right) $,
the diameter $D\left( Y\right) $ and the Gaussian width $G\left( Y\right) $.
We needn't worry about $G\left( \tciFourier \left( y_{0}\right) \right) $,
because we are free to choose $y_{0}$, so we can set it to $0$. Then $f\left(
0\right) =0$ for all $f\in \tciFourier $, whence $G\left( \tciFourier \left(
y_{0}\right) \right) =0$. For the bounds on $L\left( \tciFourier \right) $, $%
Q\left( \tciFourier \right) $, $D\left( Y\right) $ and $G\left( Y\right) $
we refer to the appendix.
\vbox{\hrule height0.6pt\hbox{\vrule height1.3ex% 
width0.6pt\hskip0.8ex\vrule width0.6pt}\hrule height0.6pt}

\section{Conclusion}

We presented a framework to derive Rademacher bounds for a wide class of vector-valued functions 
combined with Lipschitz losses. We studied in parallel the case of multi-task and multi-category learning. To our knowledge our framework allows to derive bounds for more general classes of vector-valued function and loss functions than currently possible, while still improving over existing bounds \cite{Kloft 2015,Maurer 2012} in special cases. In particular, we illustrate how bounds can be derived for neural networks with one hidden layer and rather general nonlinear activation functions. 

In the future, it would be valuable to study more examples of the loss functions included in the setting. In addition to one-vs-one classification, which we briefly mentioned in the paper, these could include multi-label classification or hybrid multi-task learning, in which each task is itself a multi-category or multi-label problem. Another interesting direction of research is to extend our analysis to neural networks with more than one hidden layer. Although the proof technique presented in Section \ref{sec:3.3} could naturally be extended to derive such bounds, it seems important to study improvement in the large constants appearing in Theorem \ref{Theorem Chain Rule1} (see \cite{Maurer 2014}) in order to avoid explosion of the constants in bounds for deep networks.

%%%%%%%%%%%%%%
%%%%%%%%%%%%%%
%%%%%%%%%%%%%%
%%%%%%%%%%%%%%
%%%%%%%%%%%%%%
%%%%%%%%%%%%%%
%%%%%%%%%%%%%%

\appendix
\section{Appendix}

For the convenience of the reader we restate in greater generality the results contained in the main body of
the paper. The $\epsilon _{i}$ or $\epsilon _{ti}$ are throughout independent Rademacher
variables.

\subsection{Mixed Norms}
In this section we prove a more general result implying Theorem \ref{Lemma lower and upper bounds frobenius}.
\begin{theorem}
\label{Lemma lower and upper bounds frobenius2}
We have that:
\begin{enumerate}
\item[(i)] For $p\in \lbrack 2,\infty
]$%
\begin{equation*}
\frac{B}{\sqrt{2}N}\sum_{t=1}^{T}\sqrt{\left\vert I_{t}\right\vert {\rm tr}( 
\hat{C}_{t}) }\leq R_{I}\left( \mathcal{W}_{2,\infty },\mathbf{x}%
\right) \leq R_{I}\left( \mathcal{W}_{2,p},\mathbf{x}\right) \leq \frac{B%
\sqrt{T}}{N}\sqrt{\sum_{t=1}^{T}\left\vert I_{t}\right\vert {\rm tr}( \hat{C}%
_{t}) }. 
\end{equation*}%
\item[(ii)] For $p\in \lbrack 1,2]$ and $1/p+1/q=1$ if $\sum_{i\in I_{t}}\left\Vert
x_{i}\right\Vert ^{2}\geq q^{-1}$ then 
\begin{equation*}
R_{I}\left( \mathcal{W}_{2,2},\mathbf{x}\right) \leq R_{I}( \mathcal{W}
_{2,p},\mathbf{x}) \leq 
\frac{T^{1/p}B\sqrt{q}}{N}
\left( 2\sum_{t}
\sqrt{\vert I_{t}\vert {\rm tr}( \hat{C}_t)}^q \right)
^{1/q}, 
\end{equation*}%
where $1/p+1/q=1$.
\item[(iii)] For 1-vs-all multi-category learning the condition $\sum_{i\in
I_{t}}\left\Vert x_{i}\right\Vert ^{2}\geq q^{-1}$ can be omitted and the
bound in (ii) can be simplified to 
\begin{equation*}
R_{I^{\rm mc}}\left( \mathcal{W}_{2,p},\mathbf{x}\right) \leq B\sqrt{%
\frac{qT~{\rm tr}( \hat{C}) }{n}}\,.
\end{equation*}
\end{enumerate}
\end{theorem}
\begin{proof}
(i) We have%
\begin{eqnarray*}
\frac{B}{\sqrt{2}}\sum_{t}\sqrt{\left\vert I_{t}\right\vert {\rm tr}( \hat{C}%
_{t}) } &=&\frac{B}{\sqrt{2}}\sum_{t}\sqrt{\sum_{i\in I_{t}}\left\Vert
x_{i}\right\Vert ^{2}}=\frac{B}{\sqrt{2}}\sum_{t}\sqrt{\mathbb{E}\left\Vert
\sum_{i\in I_{t}}\epsilon _{i}x_{i}\right\Vert ^{2}} \\
&\leq &B\sum_{t}\mathbb{E}\left\Vert \sum_{i\in I_{t}}\epsilon
_{i}x_{i}\right\Vert =\sum_{t}\mathbb{E}\sup_{w,\left\Vert w\right\Vert \leq
B}\left\langle w,\sum_{i\in I_{t}}\epsilon _{i}x_{i}\right\rangle \\
&=&N~R_{I}\left( \mathcal{W}_{2,\infty }\right) \leq N~R_{I}\left( \mathcal{W%
}_{2,p}\right) \leq N~R_{I}\left( \mathcal{W}_{2,2}\right) \\
&=&\mathbb{E}\sup_{W\in \mathcal{W}_{2,2}}\sum_{t}\left\langle
w_{t},\sum_{i\in I_{t}}\epsilon _{i}x_{i}\right\rangle =B\sqrt{T}\mathbb{E}%
\sqrt{\sum_{t}\left\Vert \sum_{i\in I_{t}}\epsilon _{i}x_{i}\right\Vert ^{2}}
\\
&\leq &B\sqrt{T}\sqrt{\sum_{t}\sum_{i\in I_{t}}\left\Vert x_{i}\right\Vert
^{2}}=B\sqrt{T}\sqrt{\sum_{t}\left\vert I_{t}\right\vert {\rm tr}( \hat{C}%
_{t}) }
\end{eqnarray*}%
where we used Szarek's inequality (Theorem 5.20 \cite{Boucheron 2013}) in
the first inequality. The next inequalities follow from $\mathcal{W}%
_{2,\infty }\subseteq \mathcal{W}_{2,p}\subseteq \mathcal{W}_{2,2}$. For the
last inequality we use Jensen's. 

(ii) The first inequality is $\mathcal{W}_{2,2}\subseteq 
\mathcal{W}_{2,p}$. Then let $X_{t}=\left\Vert \sum_{i\in I_{t}}\epsilon
_{i}x_{i}\right\Vert $, so that $\mathbb{E}X_{t}\leq \sqrt{\sum_{i\in
I_{t}}\left\Vert x_{i}\right\Vert ^{2}}$. By the bounded difference
inequality (see \cite{Boucheron 2013}) for $s\geq 0$%
\begin{equation*}
\Pr \left\{ X_{t}>\mathbb{E}X_{t}+s\right\} \leq \exp \left( \frac{-s^{2}}{%
2\sum_{i\in I_{t}}\left\Vert x_{i}\right\Vert ^{2}}\right) , 
\end{equation*}%
so with integration by parts%
\begin{eqnarray*}
\mathbb{E}\left[ X_{t}^{q}\right] &\leq &\mathbb{E}X_{t}+q\int_{0}^{\infty
}s^{q-1}\Pr \left\{ X>\mathbb{E}X+s\right\} ds^{q} \\
&\leq &\mathbb{E}X_{t}+q\int_{0}^{\infty }s^{q-1}\exp \left( \frac{-s^{2}}{%
2\sum_{i\in I_{t}}\left\Vert x_{i}\right\Vert ^{2}}\right) ds \\
&=&\mathbb{E}X_{t}+\left( \sum_{i\in I_{t}}\left\Vert x_{i}\right\Vert ^{2}%
\right)^{q/2}\left( q\int_{0}^{\infty }s^{q-1}\exp \left( \frac{-s^{2}}{2}\right)
ds\right) \\
&\leq &\left(
\sum_{i\in I_{t}}
\left\Vert x_{i}\right\Vert^{2}\right)^{1/2}+\left(%
q\sum_{i\in I_{t}}\left\Vert x_{i}\right\Vert ^{2}\right)^{q/2}\leq 2\left(%
q\sum_{i\in I_{t}}\left\Vert x_{i}\right\Vert ^{2}\right)^{q/2},
\end{eqnarray*}%
where the third inequality follows from a comparison of the integral with
the moments of the standard normal distribution, and the last follows from $%
\sum_{i\in I_{t}}\left\Vert x_{i}\right\Vert ^{2}\geq q^{-1}$. Thus%
\begin{eqnarray*}
R_{I}\left( \mathcal{W}_{2,p}\right) &=&\frac{1}{N}\mathbb{E}%
\sup_{\left\Vert W\right\Vert _{2,p}\leq T^{1/p}B}\sum_{t}\sum_{i\in
I_{t}}\left\langle w_{t},x_{i}\right\rangle =\frac{T^{1/p}B}{N}\mathbb{E}%
\left( \sum_{t}X_{t}^{q}\right) ^{1/q} \\
&\leq &\frac{T^{1/p}B}{N}\left( \sum_{t}\mathbb{E}X_{t}^{q}\right)
^{1/q}\leq \frac{T^{1/p}B\sqrt{q}}{N}\left( 2\sum_{t}\left(\sum_{i\in
I_{t}}\left\Vert x_{i}\right\Vert ^{2}\right)^{q/2}\right) ^{1/q} \\
&=&\frac{2^{1/q}T^{1/p}B\sqrt{q}}{N}
\left( \sum_{t}\left(\left\vert
I_{t}\right\vert {\rm tr}( \hat{C}_{t})\right)^{q/2}\right) ^{1/q}.
\end{eqnarray*}

(iii) The case of 1-vs-all multi-category learning is simpler because $%
I_{t}=\left\{ 1,\dots,N\right\} $ and we can interchange summation over $t$
and $i$. Then we can essentially proceed as in\cite{Kloft 2015} and use the $%
1/q$-strong convexity of $\frac{1}{2}\left\Vert W\right\Vert _{2,p}^{2}$
w.r.t. $\left\Vert W\right\Vert _{2,p}$. In Corollary 4 of \cite{KakadeEtAl
2012} let $\lambda >0$ and $u=W$ and $v_{i}=\lambda \left( \epsilon
_{1i}x_{i},\dots,\epsilon _{Ti}x_{i}\right) $ and use $\frac{1}{2}\left\Vert
W\right\Vert _{2,p}^{2}\leq \frac{1}{2}\left( T^{1/p}B\right) ^{2}=f_{\max
}\left( u\right) $ to obtain%
\begin{equation*}
\sum_{i=1}^{N}\left\langle W,\lambda \left( \epsilon _{1i}x_{i},...,\epsilon
_{Ti}x_{i}\right) \right\rangle _{2}\leq \sum_{i=1}^{N}\left\langle \nabla
f\left( v_{1:i-1}\right) ,v_{i}\right\rangle +\frac{1}{2}\left(
T^{1/p}B\right) ^{2}+\frac{q\lambda ^{2}}{2}\sum_{i=1}^{N}\left\Vert \left(
\epsilon _{1i}x_{i},...,\epsilon _{Ti}x_{i}\right) \right\Vert _{2q}^{2}, 
\end{equation*}%
where $\left\langle \cdot,\cdot\right\rangle _{2}$ is the Hilbert-Schmidt inner
product. Take the supremum in $W$ and then the expectation. The first term
on the r.h.s. above vanishes. Dividing by $\lambda $ and optimizing in $%
\lambda $ gives 
\begin{equation*}
\mathbb{E}\sup_{W}\sum_{i=1}^{n}\left\langle W,\epsilon
_{1i}x_{i},\dots,\epsilon _{Ti}x_{i}\right\rangle \leq \left( T^{1/p}B\right) 
\sqrt{q\sum_{i=1}^{n}\mathbb{E}\left\Vert \left( \epsilon
_{1i}x_{i},\dots,\epsilon _{Ti}x_{i}\right) \right\Vert _{2q}^{2}}.
\end{equation*}%
Now 
\begin{equation*}
\mathbb{E}\left\Vert \left( \epsilon _{1i}x_{i},\dots,\epsilon
_{Ti}x_{i}\right) \right\Vert _{2q}^{2}=\mathbb{E}\left( \sum_{t}\left\Vert
\epsilon _{ti}x_{i}\right\Vert ^{q}\right) ^{2/q}\leq T^{2/q}\left\Vert
x_{i}\right\Vert ^{2} 
\end{equation*}%
so%
\begin{equation*}
R_{I^{\text{mc}}}\left( \mathcal{W}_{2,p}\right) =\frac{1}{N}\mathbb{E}%
\sum_{i=1}^{N}\left\langle W,\epsilon _{1i}x_{i},\dots,\epsilon
_{Ti}x_{i}\right\rangle \leq \frac{TB}{N}\sqrt{q\sum_{i=1}^{N}\left\Vert
x_{i}\right\Vert ^{2}}=B\sqrt{\frac{qT~{\rm tr}( \hat{C}) }{n}}. 
\end{equation*}
\end{proof}

Note that the (very harmless) condition $\sum_{i\in I_{t}}\left\Vert
x_{i}\right\Vert ^{2}\geq q^{-1}$ in part (iii) is automatically satisfied if $\left\Vert
x_{i}\right\Vert =1$.
\subsection{Trace Norm Constraints}

In this section we prove the following result, which contains Theorem \ref{Theorem Tracenorm} as as special case and improves over \cite{Maurer 2012} which only applies to the multi-task learning setting.
\begin{theorem}
\label{Theorem general tracenorm}%
\begin{equation*}
R_{I}\left( \mathcal{W}_{tr},\mathbf{x}\right) \leq \frac{B}{N}\sqrt{%
2T\max_{t}\left\vert I_{t}\right\vert {\rm tr}( \hat{C}_{t}) \left( \ln
N+1\right) }+\frac{B}{N}\sqrt{T~\lambda _{\max }\left( \sum_{t}\left\vert
I_{t}\right\vert \hat{C}_{t}\right) }. 
\end{equation*}
\end{theorem}

For the proof we use $\left\Vert .\right\Vert _{\infty }$ to denote the
operator norm on $H$ and $\succeq $ and $\preceq $ to refer to the ordering
induced by the cone of positive operators. For $x\in H$ we define the rank-1
operator $Q_{x}$ on $H$ by $Q_{x}v=\left\langle v,x\right\rangle x$. We use
the following result, the proof of which can be found in \cite{Maurer 2012}.

\begin{theorem}
\label{Theorem Main Tool}Let $M\subseteq H$ be a subspace of dimension $d$
and suppose that $A_{1},\dots ,A_{N}$ are independent random operators
satisfying $A_{k}\succeq 0$, $Ran\left( A_{k}\right) \subseteq M$ a.s. and 
\begin{equation*}
\mathbb{E}A_{k}^{m}\preceq m!R^{m-1}\mathbb{E}A_{k} 
\end{equation*}%
for some $R\geq 0$, all $m\in 
%TCIMACRO{\U{2115} }%
%BeginExpansion
\mathbb{N}
%EndExpansion
$ and all $k\in \left\{ 1,\dots ,N\right\} $. Then%
\begin{equation*}
\sqrt{\mathbb{E}\left\Vert \sum_{k}A_{k}\right\Vert _{\infty }}\leq \sqrt{%
\left\Vert \mathbb{E}\sum_{k}A_{k}\right\Vert _{\infty }}+\sqrt{R\left( \ln
\dim \left( M\right) +1\right) }. 
\end{equation*}%
\bigskip
\end{theorem}

\begin{lemma}
\label{Lemma Subexbound} Let $x_{1},\dots,x_{n}$ be in $
\mathbb{R}^{d}$ and denote%
\begin{equation*}
\alpha =\sum_{i=1}^{n}\left\Vert x_{i}\right\Vert ^{2}. 
\end{equation*}%
Define a random vector by $V=\sum_{i}\epsilon _{i}x_{i}$. Then for $p\geq 1$%
\begin{equation*}
\mathbb{E}\left[ Q_{V}^{p}\right] \preceq \left( 2p-1\right) !!\alpha
^{p-1} \mathbb{E}\left[ Q_{V}\right], 
\end{equation*}%
where $\left( 2p-1\right) !!=\prod_{i=1}^{p}\left( 2i-1\right) =\left(
2p-1\right) \left( 2\left( p-1\right) -1\right) \times \dots\times 5\times
3\times 1$.
% and $\preceq $ refers to the ordering of symmetric operators by the positive semi-definite cone.
\end{lemma}
\begin{proof}
Let $v\in 
%TCIMACRO{\U{211d} }%
%BeginExpansion
\mathbb{R}
%EndExpansion
^{d}$ be arbitrary. By the definition of $V$ and $Q_{V}$ we have for any $%
v\in 
%TCIMACRO{\U{211d} }%
%BeginExpansion
\mathbb{R}
%EndExpansion
^{d}$ that%
\begin{equation*}
\left\langle \mathbb{E}\left[ Q_{V}^{p}\right] v,v\right\rangle
=\sum_{j_{1},\dots,j_{2p}=1}^{n}\mathbb{E}\left[ \epsilon _{j_{1}}\epsilon
_{j_{2}}\cdots \epsilon _{j_{2p}}\right] \left\langle v,x_{j_{1}}\right\rangle
\left\langle x_{j_{1}},x_{j_{2}}\right\rangle \cdots \left\langle
x_{j_{2p}},v\right\rangle . 
\end{equation*}%
The properties of independent Rademacher variables imply that $\mathbb{E}%
\left[ \epsilon _{i_{1}}\epsilon _{i_{2}} \cdots\epsilon _{i_{2p}}\right] =0$
unless the sequence $\mathbf{i}=\left( i_{1},\dots,i_{2p}\right) $ has the
property that each index $i_{k}$ occurs in it an even number of times, in
which case $\mathbb{E}\left[ \epsilon _{i_{1}}\epsilon _{i_{2}} \cdots \epsilon
_{i_{2p}}\right] =1$. Let us call sequences with this property admissible.
Thus%
\begin{eqnarray*}
\left\langle E\left[ Q_{w}^{p}\right] v,v\right\rangle &=&\sum_{\mathbf{i}%
\text{ admissible}}\left\langle v,x_{i_{1}}\right\rangle \left\langle
x_{i_{2}},x_{i_{3}}\right\rangle \cdots \left\langle x_{i_{2p}},v\right\rangle \\
&\leq &\sum_{\mathbf{i}\text{ admissible}}\left\vert \left\langle
v,x_{i_{1}}\right\rangle \right\vert \prod_{k=2}^{2p-1}\left\Vert
x_{i_{k}}\right\Vert \left\vert \left\langle x_{i_{2p}},v\right\rangle
\right\vert ,
\end{eqnarray*}%
using Cauchy-Schwarz. For every admissible sequence $\mathbf{i}$ there
exists at least one partition $\mathbf{\pi }$ of $\left\{ 1,..,2p\right\} $
into $p$ pairs $\left( l,r\right) $ with $l<r$, such that the indices $%
i_{k_{1}}$ and $i_{k_{2}}$ are equal, whenever $k_{1}$ and $k_{2}$ belong to
the same pair. Let us denote the latter condition by $\mathbf{i}\sim \mathbf{%
\pi }$. It is easy to show by induction that there are $\left( 2p-1\right)
!! $ such partitions into pairs. Given $\mathbf{\pi }$ we can write $\left\{
1,\dots,2p\right\} =L_{\pi }\cup R_{\pi }$, where $L_{\pi }=\left\{ l:\exists
\left( l,r\right) \in \pi \right\} $ and $R_{\pi }=\left\{ r:\exists \left(
l,r\right) \in \pi \right\} $. We always have $1\in L_{\pi }$ and $2p\in
R_{\pi }$ and $\left\vert L_{\pi }\right\vert =\left\vert R_{\pi
}\right\vert =p$. Thus 
\begin{eqnarray*}
\left\langle E\left[ Q_{w}^{p}\right] v,v\right\rangle &\leq &\sum_{\mathbf{%
\pi }}\sum_{\mathbf{i}\sim \mathbf{\pi }}\left\vert \left\langle
v,x_{i_{1}}\right\rangle \right\vert \prod_{k=2}^{2p-1}\left\Vert
x_{i_{k}}\right\Vert \left\vert \left\langle x_{i_{2p}},v\right\rangle
\right\vert \\
&=&\sum_{\mathbf{\pi }}\sum_{\mathbf{i}\sim \mathbf{\pi }}\left( \left\vert
\left\langle v,x_{i_{1}}\right\rangle \right\vert \prod_{k=2,i_{k}\in L_{\pi
}}^{2p-1}\left\Vert x_{i_{k}}\right\Vert \right) \left( \left\vert
\left\langle x_{i_{2p}},v\right\rangle \right\vert \prod_{k=2,i_{k}\in
R_{\pi }}^{2p-1}\left\Vert x_{i_{k}}\right\Vert \right) \\
&\leq &\sum_{\mathbf{\pi }}\sum_{\mathbf{i}\sim \mathbf{\pi }}\left\langle
v,x_{i_{1}}\right\rangle ^{2}\prod_{k=2,i_{k}\in L_{\pi }}^{2p-1}\left\Vert
x_{i_{k}}\right\Vert ^{2}.
\end{eqnarray*}%
The last step follows from the Cauchy-Schwarz inequality and realizing that
the two resulting factors are equal by symmetry. But for $\mathbf{i}\sim 
\mathbf{\pi }$ we just need to sum over the indices in $L_{\pi }$, the
others being constrained to be equal. Thus, writing $L_{\pi }=\left\{
l_{1},\dots,l_{p}\right\} $ such that $l_{1}=1$ the last expression above is
just%
\begin{eqnarray*}
&&\sum_{\mathbf{\pi }}\sum_{i_{1},\dots,i_{p}}\left\langle
v,x_{i_{1}}\right\rangle ^{2}\prod_{k=2}^{p}\left\Vert x_{i_{k}}\right\Vert
^{2} \\
&=&\left( 2p-1\right) !!\left( \sum_{i=1}^{n}\left\Vert x_{i}\right\Vert
^{2}\right) ^{p-1}\left\langle \sum_{i=1}^{n}Q_{x_{i}}v,v\right\rangle \\
&=&\left( 2p-1\right) !!\left( \sum_{i=1}^{n}\left\Vert x_{i}\right\Vert
^{2}\right) ^{p-1}\left\langle \mathbb{E}\left[ Q_{V}\right] v,v\right\rangle\,.
\end{eqnarray*}%
The conclusion follows since for symmetric matrices $\left( \forall
v,\left\langle Av,v\right\rangle \leq \left\langle B,v,v\right\rangle
\right) \implies A\preceq B$.
\end{proof}

\noindent {\it Proof of Theorem \protect\ref{Theorem general tracenorm}.}~
We have%
\begin{equation*}
R_{I}\left( \mathcal{W}_{tr},\mathbf{x}\right) =\frac{1}{N}\mathbb{E}%
\sup_{W\in \mathcal{W}_{\rm tr}}\sum_{t}\sum_{i\in I_{t}}\epsilon
_{ti}\left\langle w_{t},x_{i}\right\rangle =\frac{1}{N}\mathbb{E}\sup_{W\in 
\mathcal{W}_{\rm tr}}{\rm tr}( W^{\ast }D) , 
\end{equation*}%
where the random operator $D:H\rightarrow 
%TCIMACRO{\U{211d} }%
%BeginExpansion
\mathbb{R}
%EndExpansion
^{T}$ is defined for $v\in H$ by $\left( Dv\right) _{t}=\left\langle
v,\sum_{i\in I_{t}}\epsilon _{ti}x_{i}\right\rangle $. H\"{o}lder's inequality gives
\begin{equation*}
R_{I}\left( \mathcal{W}_{\rm tr},\mathbf{x}\right) \leq \frac{B\sqrt{T}}{N}%
\mathbb{E}\left\Vert D\right\Vert _{\infty }. 
\end{equation*}%
We proceed to bound $\mathbb{E}\left\Vert D\right\Vert _{\infty }$. Let $%
V_{t}$ be the random vector $V_{t}=\sum_{i\in I_{t}}^{n_{t}}\epsilon
_{ti}x_{i}$ and recall that the corresponding rank-one operator $Q_{V_{t}}$
is defined by $Q_{V_{t}}v=\left\langle v,V_{t}\right\rangle
V_{t}=\left\langle v,\sum_{i\in I_{t}}^{n_{t}}\epsilon
_{ti}x_{i}\right\rangle \sum_{i\in I_{t}}^{n_{t}}\epsilon _{ti}x_{i}$. Then $%
D^{\ast }D=\sum_{t=1}^{T}Q_{V_{t}}$, so by Jensen's inequality%
\begin{equation*}
\mathbb{E}\left\Vert D\right\Vert _{\infty }\leq \sqrt{\mathbb{E}\left\Vert
\sum_{t}Q_{V_{t}}\right\Vert _{\infty }}. 
\end{equation*}%
The range of any of the realizations of $Q_{V_{t}}$ lies in the span of the $%
x_{i}$ which has less than $N$. By Lemma \ref{Lemma Subexbound} we have
with $\alpha _{t}=\sum_{i\in I_{t}}\left\Vert x_{i}\right\Vert ^{2}$%
\begin{equation*}
\mathbb{E}\left[ \left( Q_{Vt}\right) ^{m}\right] \preceq \left( 2p-1\right)
!!\alpha _{t}^{m-1}\mathbb{E}\left[ Q_{V_{t}}\right] \preceq m!\left(
2\max_{t}\alpha _{t}\right) ^{m-1}\mathbb{E}\left[ Q_{V_{t}}\right] , 
\end{equation*}%
so Theorem \ref{Theorem Main Tool} with $R=2\max_{t}\alpha _{t}$ and $d=N$
now gives%
\begin{equation*}
\sqrt{\mathbb{E}\left\Vert \sum_{t}Q_{V_{t}}\right\Vert _{\infty }}\leq 
\sqrt{2\max_{t}\alpha _{t}\left( \ln N+1\right) }+\sqrt{\left\Vert \mathbb{E}%
\sum_{t}Q_{V_{t}}\right\Vert _{\infty }}. 
\end{equation*}%
But $\mathbb{E}\left[ Q_{V_{t}}\right] =\sum_{i\in
I_{t}}Q_{x_{i}}=\left\vert I_{t}\right\vert \hat{C}_{t}$, so%
\begin{eqnarray*}
R_{I}\left( \mathcal{W}_{\rm tr},\mathbf{x}\right) &\leq &\frac{B\sqrt{T}}{N}%
\mathbb{E}\left\Vert D\right\Vert _{\infty }\leq \frac{B\sqrt{T}}{N}\sqrt{%
\mathbb{E}\left\Vert \sum_{t}Q_{V_{t}}\right\Vert _{\infty }} \\
&\leq &\frac{B}{N}\sqrt{2T\max_{t}\left\vert I_{t}\right\vert {\rm tr}( \hat{C%
}_{t}) \left( \ln N+1\right) }+\sqrt{T\left\Vert \sum_{t}\left\vert
I_{t}\right\vert \hat{C}_{t}\right\Vert _{\infty }.}
\end{eqnarray*}
\vbox{\hrule height0.6pt\hbox{\vrule height1.3ex% 
width0.6pt\hskip0.8ex\vrule width0.6pt}\hrule height0.6pt}

\subsection{Nonlinear Compositions}

For the statement of a general version of Theorem \ref{Theorem composite class} we extend the definition
of $\theta _{\text{mc}}$ and $\theta _{\text{mt}}$ by setting for any map $%
I:\left\{ 1,\dots,T\right\} \rightarrow 2^{\left\{ 1,\dots,N\right\} }$%
\begin{equation*}
\theta _{I}=\inf \left\{ \theta :\forall \left( a_{1},\dots,a_{N}\right)
,a_{i}\geq 0,~\sum_{t=1}^{T}\sum_{i\in I_{t}}a_{i}\leq \theta
^{2}\sum_{i=1}^{N}a_{i}\right\} . 
\end{equation*}%
This definition coincides with the previous one in the case of multi-task
and 1-vs-all multi-category learning.

\begin{theorem}
\label{Theorem general composite class}There are universal constants $c_{1}$
and $c_{2}$ such that under the above conditions%
\begin{eqnarray*}
R_{I}\left( \mathcal{V\phi }\left( \mathcal{W}_{2,\infty }\right) ,\mathbf{x}%
\right)  &\leq &L_{\phi }ab_{\infty }\theta _{I}\left( c_{1}K\sqrt{\frac{%
{\rm tr}(\hat{C})}{nT}}+c_{2}\sqrt{\frac{K\lambda _{\max }(\hat{C})}{n}}%
\right)  \\
R_{I}\left( \mathcal{V\phi }\left( \mathcal{W}_{2,2}\right) ,\mathbf{x}%
\right)  &\leq &L_{\phi }ab_{2}\theta _{I}\left( c_{1}\sqrt{\frac{K~{\rm tr}%
(\hat{C})}{nT}}+c_{2}\sqrt{\frac{\lambda _{\max }\left( \hat{C}\right) }{n}}%
\right)  \\
R_{I}\left( \mathcal{V\phi }\left( \mathcal{W}_{2,1}\right) ,\mathbf{x}%
\right)  &\leq &L_{\phi }ab_{1}\theta _{I}\left( c_{1}\sqrt{\frac{2{\rm tr}(%
\hat{C})+8\lambda _{\max }(\hat{C})\ln K}{nT}}+c_{2}\sqrt{\frac{\lambda
_{\max }(\hat{C})}{n}}\right).
\end{eqnarray*}
\end{theorem}

The proof uses the following recent result on the expected suprema of
Gaussian processes \cite{Maurer 2014}. For a set $Y\subseteq \mathbb{R}^{m}$
the Gaussian width $G\left( Y\right) $ is defined as 
\begin{equation*}
G\left( Y\right) =\mathbb{E}\sup_{y\in Y}\left\langle \gamma ,y\right\rangle
=\mathbb{E}\sup_{y\in Y}\sum_{i=1}^{m}\gamma _{i}y_{i}, 
\end{equation*}%
where $\gamma =\left( \gamma _{1},\dots ,\gamma _{m}\right) $ is a vector of
independent standard normal variables.

\begin{theorem}
\label{Theorem Chain Rule}Let $Y\subseteq \mathbb{R}^{n}$ have (Euclidean)
diameter $D\left( Y\right) $ and let $\tciFourier $ be a class of functions $%
f:Y\rightarrow \mathbb{R}^{m}$, all of which have Lipschitz constant at most 
$L\left( \tciFourier \right) $. Let $\tciFourier \left( Y\right) =\left\{
f\left( y\right) :f\in \tciFourier ,y\in Y\right\} $. Then for any
\thinspace $y_{0}\in Y$ 
\begin{equation}
G\left( \tciFourier \left( Y\right) \right) \leq c_{1}L\left( \tciFourier
\right) G\left( Y\right) +c_{2}W\left( Y\right) Q\left( \tciFourier \right)
+G\left( \tciFourier \left( y_{0}\right) \right) ,  \label{eq:jjj}
\end{equation}%
where $c_{1}$ and $c_{2}$ are universal constants and 
\begin{equation*}
Q\left( \tciFourier \right) =\sup_{\mathbf{y},\mathbf{y}^{\prime }\in Y,~%
\mathbf{y}\neq \mathbf{y}^{\prime }}\mathbb{E}\sup_{f\in \tciFourier }\frac{%
\left\langle \mathbf{\gamma },f\left( \mathbf{y}\right) -f\left( \mathbf{y}%
^{\prime }\right) \right\rangle }{\left\Vert \mathbf{y}-\mathbf{y}^{\prime
}\right\Vert }. 
\end{equation*}
\end{theorem}

\noindent {\it Proof of Theorem \protect\ref{Theorem composite class}.~}
We will use Theorem \ref{Theorem Chain Rule} by setting
\[Y=\left\{ W\mathbf{x%
}=\left( \langle w_{k},x_{i}\rangle \right) _{k\leq K,~i\leq N}:W\in 
\mathcal{W}\right\} \subseteq \mathbb{R}^{KN}
\] 
where $\mathcal{W}$ will be
either $\mathcal{W}_{2,\infty }$, $\mathcal{\ W}_{2,2}$ or $\mathcal{W}_{2,1}$.
For $\tciFourier $ we take the set of functions 
\begin{equation*}
\left\{ \left( y_{ki}\right) \in \mathbb{R}^{KN}\mapsto \left( \left\langle
v_{t},\phi \left( y_{i}\right) \right\rangle \right) _{t\leq T,i\in
I_{t}}\in \prod_{t=1}^{T}\mathbb{R}^{\left\vert I_{t}\right\vert }:v\in 
\mathcal{V}\right\} 
\end{equation*}%
restricted to $Y$.
%, so $\tciFourier \left( Y\right) $ is a subset of $\mathbb{R}^{T^{2}n}$ for 1-vs-all multi-category and $\mathbb{R}^{Tn}$ for multi-task learning, which again accounts for the additional factor of $%
%\sqrt{T}$ for the complexity of multi-category learning. 
By a well known bound
on Rademacher averages in terms of Gaussian averages \cite{Ledoux Talagrand
1991} 
\begin{eqnarray}
\nonumber 
\mathbb{E}\sup_{W\in \mathcal{V},W\in \mathcal{W}}\sum_{t}\sum_{i\in
I_{t}}\epsilon _{ti}V\phi \left( Wx_{i}\right) & \leq & \sqrt{\frac{\pi }{2}}%
\mathbb{E}\sup_{W\in \mathcal{V},W\in \mathcal{W}}\sum_{t}\sum_{i\in
I_{t}}\gamma _{ti}V\phi \left( Wx_{i}\right) \\
& =& \sqrt{\frac{\pi }{2}}G\left(
\tciFourier \left( Y\right) \right).  
\label{Bound by Gaussian width}
\end{eqnarray}
To bound $G\left( \tciFourier \left( Y\right) \right) $ we then just need to
bound the terms in the right hand side of equation \eqref{eq:jjj}

Since $\phi \left( 0\right) =0$, we can at once set $G\left( \tciFourier
\left( y_{0}\right) \right) =0$, by setting $0=y_{0}$, so $f\left( 0\right)
=0$ for all $f\in \tciFourier $.

\vspace{.2truecm}
\noindent {\em Bounding the Lipschitz constant.} For any $v\in \mathcal{V}$ and $y,y^{\prime
}\in Y\subseteq \mathbb{R}^{KN}$, 
\begin{eqnarray*}
\sum_{t,i\in I_{t}}\left( \left\langle v_{t},\phi \left( y_{i}\right)
\right\rangle -\left\langle v_{t},\phi \left( y_{i}^{\prime }\right)
\right\rangle \right) ^{2} &\leq &\sum_{t}\left\Vert v_{t}\right\Vert
^{2}\sum_{i\in I_{t}}\left\Vert \phi \left( y_{i}\right) -\phi \left(
y_{i}^{\prime }\right) \right\Vert ^{2} \\
&\leq &a^{2}L_{\phi }^{2}\sum_{t}\sum_{i\in I_{t}}\left\Vert
y_{i}-y_{i}^{\prime }\right\Vert ^{2}\leq a^{2}L_{\phi }^{2}\theta
_{I}^{2}\left\Vert y-y^{\prime }\right\Vert ^{2},
\end{eqnarray*}%
so $L\left( \tciFourier \right) \leq aL_{\phi }\theta _{I}$\textbf{.}

\vspace{.2truecm}
\noindent {\em Bounding $Q\left( \tciFourier \right) $}. Again with $y,y^{\prime }\in Y$ 
\begin{align*}
& \mathbb{E}\sup_{f\in \tciFourier }\left\langle \mathbf{\gamma },f\left( 
\mathbf{y}\right) -f\left( \mathbf{y}^{\prime }\right) \right\rangle \\
& =\mathbb{E}\sup_{v\in \mathcal{V}}\sum_{ti}\gamma _{ti}\left( \left\langle
v_{t},\phi \left( y_{i}\right) \right\rangle -\left\langle v_{t},\phi \left(
y_{i}^{\prime }\right) \right\rangle \right) =\mathbb{E}\sup_{v\in \mathcal{V%
}}\sum_{t}\left\langle v_{t},\sum_{i\in I_{t}}\gamma _{ti}\left( \phi \left(
y_{i}\right) -\phi \left( y_{i}^{\prime }\right) \right) \right\rangle \\
& \leq a\mathbb{E}\sum_{t}\left\Vert \sum_{i\in I_{t}}\gamma _{ti}\left(
\phi \left( y_{i}\right) -\phi \left( y_{i}^{\prime }\right) \right)
\right\Vert \leq \sqrt{T}a\left( \sum_{t}E\left\Vert \sum_{i\in I_{t}}\gamma
_{ti}\left( \phi \left( y_{i}\right) -\phi \left( y_{i}^{\prime }\right)
\right) \right\Vert ^{2}\right) ^{1/2} \\
& \leq aL_{\phi }\sqrt{T}\left( \sum_{t}\sum_{i\in I_{t}}\left\Vert
y_{i}-y_{i}^{\prime }\right\Vert ^{2}\right) ^{1/2}\leq aL_{\phi }\theta _{I}%
\sqrt{T}\left\Vert y-y^{\prime }\right\Vert ,
\end{align*}%
so $Q\left( \tciFourier \right) \leq aL_{\phi }\theta _{I}\sqrt{T}$.

\vspace{.2truecm}
\noindent {\em Bounding the diameters.} We have 
\begin{eqnarray*}
D\left( \mathcal{W}\mathbf{x}\right) &\leq &2\sqrt{\sup_{W}\sum_{ki}\left%
\langle w_{k},x_{i}\right\rangle ^{2}}=\sqrt{\sup_{W}\sum_{k}\left\Vert
w_{k}\right\Vert ^{2}\sum_{i}\left\langle \frac{w_{k}}{\left\Vert
w_{k}\right\Vert },x_{i}\right\rangle ^{2}} \\
&\leq &\sqrt{\sup_{W}\sum_{k}\left\Vert w_{k}\right\Vert ^{2}N\lambda _{\max
}(\hat{C})}=\left\Vert W\right\Vert _{2,2}\sqrt{N\lambda _{\max }(\hat{C})}.
\end{eqnarray*}%
From $\left\Vert W\right\Vert _{2,2}\leq \left\Vert W\right\Vert _{2,1}$ and 
$\left\Vert W\right\Vert _{2,2}\leq \sqrt{K}\left\Vert W\right\Vert
_{2,\infty }$ we obtain 
\begin{equation*}
D\left( \mathcal{W}_{2,\infty }\right) \leq b_{\infty }\sqrt{KN\lambda _{\max
}(\hat{C})}\text{, and both }D\left( \mathcal{W}_{2,2}\right) ,D\left( 
\mathcal{W}_{2,1}\right) \leq b_{2}\sqrt{N\lambda _{\max }(\hat{C})}. 
\end{equation*}

\vspace{.2truecm}
\noindent {\em Bounding the Gaussian width.} 
\begin{equation*}
G\left( \mathcal{W}_{2,\infty }\mathbf{x}\right) =\mathbb{E}\sup_{W\in 
\mathcal{W}_{\infty }}\sum_{k}\left\langle w_{k},\sum_{i\leq N}\gamma
_{ki}x_{i}\right\rangle =b_{\infty }\sum_{k}\mathbb{E}\left\Vert \sum_{i\leq
N}\gamma _{ki}x_{i}\right\Vert \leq b_{\infty }K\sqrt{N{\rm tr}(\hat{C})}. 
\end{equation*}%
similarly 
\begin{equation*}
G\left( \mathcal{W}_{2,2}\mathbf{x}\right) =\mathbb{E}\sup_{W\in \mathcal{W}%
_{2}}\sum_{k}\left\langle w_{k},\sum_{i\leq N}\gamma _{ki}x_{i}\right\rangle
=b_{2}\sqrt{\sum_{k}\mathbb{E}\left\Vert \sum_{i\leq N}\gamma
_{ki}x_{i}\right\Vert ^{2}}\leq b_{\infty }\sqrt{KN~{\rm tr}(\hat{C})}. 
\end{equation*}%
The Gaussian width of $\mathcal{W}_{1}\mathbf{x}$ is a little more
complicated. Let $\mathcal{W}_{1}^{\left( k\right) }$ be the class of linear
transformations $\mathcal{W}_{1}^{\left( k\right) }=\left\{ x\mapsto \left(
0,\dots ,\left\langle w,x\right\rangle ,\dots ,0\right) :\left\Vert
w\right\Vert \leq b_{1}\right\} $, where only the $k$-th coordinate is
different from zero. Then $\mathcal{W}_{1}\mathbf{x}$ is the convex hull of $%
\mathcal{W}_{1}^{\left( 1\right) }\mathbf{x}\cup \dots \cup \mathcal{W}%
_{1}^{\left( K\right) }\mathbf{x}$. It follows from Lemma 2 in \cite{Maurer
2014colt} that 
\begin{eqnarray*}
G\left( \mathcal{W}_{2,1}\mathbf{x}\right) &\leq &\max_{k}G\left( \mathcal{W}%
_{1}^{\left( k\right) }\mathbf{x}\right) +2\sqrt{\sum_{k,i}\left\langle
w_{k},x_{i}\right\rangle ^{2}\ln K} \\
&\leq &b_{1}\sqrt{N~{\rm tr}(\hat{C})}+2\sqrt{\sum_{k}\left\Vert
w_{k}\right\Vert ^{2}\sum_{i}\left\langle \frac{w_{k}}{\left\Vert
w_{k}\right\Vert },x_{i}\right\rangle ^{2}\ln K} \\
&\leq &b_{1}\sqrt{N~{\rm tr}(\hat{C})}+2b_{1}\sqrt{N\lambda _{\max }(\hat{C}%
)\ln K} \\
&\leq &b_{1}\sqrt{2N~\left( {\rm tr}(\hat{C})+8\lambda _{\max }\left( \hat{C%
}\right) \ln K\right) }.
\end{eqnarray*}%
Collecting these bounds in Theorem \ref{Theorem Chain Rule} and using (\ref%
{Bound by Gaussian width}) gives the three inequalities of Theorem \ref%
{Theorem composite class}.

\bigskip \vbox{\hrule height0.6pt\hbox{\vrule height1.3ex% 
width0.6pt\hskip0.8ex\vrule width0.6pt}\hrule height0.6pt}

\end{document}